\pdfoutput=1

\documentclass{article}
\usepackage{spconf,amsmath,epsfig}
\usepackage{amssymb}
\usepackage{amsthm}
\usepackage{algorithm}
\usepackage{algpseudocode}
\usepackage{tikz}
\usepackage{graphicx}
\usepackage{subfig}
\usepackage{epstopdf}
\usepackage{epsfig}
\usetikzlibrary{tikzmark,calc}
\usepackage{caption}
\usepackage{float}
\usepackage{multirow}
\usepackage{todonotes}
\usepackage{microtype}
\usepackage{hyperref}
\usepackage{mathtools}
\usepackage{pgfplots}
\usepackage{rotating}
\numberwithin{equation}{section} 
\newtheorem{theorem}{Theorem}[section]                   %[section]   %numberes automatically
\newtheorem{definition}[theorem]{Definition}

\newcommand{\SREC}{S-REC}
\DeclareMathOperator*{\argmin}{arg\,min}

\def\R{{\mathbb{R}}}
\def\R{{\mathbb{R}}}

\def\S{{\mathcal{S}}}

\newcommand{\s}{0.060}

\title{Solving Linear Inverse Problems Using GAN Priors: \\ An Algorithm with Provable Guarantees}
\name{Viraj Shah and Chinmay Hegde \thanks{This work was supported in part by grants from the National Science Foundation (NSF CCF-1566281) and NVIDIA.}
}

\address{ECpE Department, Iowa State University, Ames, IA, 50010}
 
\begin{document}
	\maketitle
	\ninept
%	\todo{check the acknowledgement information}
	\begin{abstract}
In recent works, both sparsity-based methods as well as learning-based methods have proven to be successful in solving several challenging linear inverse problems. However, sparsity priors for natural signals and images suffer from poor discriminative capability, while learning-based methods seldom provide concrete theoretical guarantees. In this work, we advocate the idea of replacing hand-crafted priors, such as sparsity, with a Generative Adversarial Network (GAN) to solve linear inverse problems such as compressive sensing. In particular, we propose a projected gradient descent (PGD) algorithm for effective use of GAN priors for linear inverse problems, and also provide theoretical guarantees on the rate of convergence of this algorithm. Moreover, we show empirically that our algorithm demonstrates superior performance over an existing method of leveraging GANs for compressive sensing.   
%\red{revise once}
\end{abstract}

	\section{Introduction}

\subsection{Motivation}
Linear inverse problems arise in diverse range of application domains such as computational imaging, optics, astrophysics, and seismic geo-exploration. Formally put, the basic structure of a linear inverse problem can be represented in terms of a linear equation of the form: 
\begin{align}
y = Ax^*+e,~~\label{eq:lip}
\end{align}
where $x^* \in \R^n$ is the target signal (or image), $A \in \R^{m \times n}$ is a linear operator that captures the forward process, $y \in \R^m$ denotes the given observations, and $e \in \R^m$ represents stochastic noise. The aim is to recover (an estimate of) the unknown signal $x^*$ given $y$ and $A$. 

Many important problems in signal and image processing can be modeled as linear inverse problems. For example, the classical problem of \emph{super-resolution} corresponds to the case where the operator $A$ represents a low-pass filter followed by downsampling. The problem of \emph{image inpainting} corresponds to the case where $A$ can be modeled as a pixel-wise selection operator applied to the original image. %Image denoising is also an inverse problem with $A$ being identity matrix. 
%Our focus in this paper is Compressive Sensing, in which the matrix $A$ is a short-fat matrix with fewer rows than columns. \emph{Compression} is a process of reducing the size of the data significantly by throwing away its unimportant part and retaining only the valuable portion. \emph{Compressed Sensing} aims to introduce the compression at the image acquisition stage itself, \textit{i.e.}, acquire only the valuable part of information from the scene, and use it to reconstruct the original image. Thus, we have $A \in \R^{m \times n}$ with $m<n$. 
Similar challenges arise in image denoising as well as compressive sensing~\cite{candes2006compressive,candes2006stable,donoho2006compressed}. 

In general, when $m < n$ the inverse problem is ill-posed.
%	Compressed sensing aims to solve the problem of reconstructing an unknown $n-$ dimensional vector from its $m$ linear measurements obtained in under-determined setting; \textit{i.e.} when $m < n$. Formally put, aim is to find $x^*$ in the following equation:
%	
%	where $x \in \R^n, y \in \R^m$. $A \in \R^{m \times n}$ is called measurement matrix.
%The problem of recovering $x^*$ given $y$ and $A$ is ill-posed and thus is extremely difficult to solve since the matrix $A$ has a non-trivial null space; \textit{i.e.}, there exists an infinite number of feasible solutions, out of which only a few are valid signals. Thus, an additional piece of information about structure of the $x^*$ is required to act as a regularization for the accurate recovery. 
A common approach for resolving this issue is to obtain an estimate of $x^*$ as the solution to the constrained optimization problem:
\begin{align}
\widehat{x} &= \argmin~f(y; Ax),~~\label{eq:cop}\\
&\text{s. t.}~~~x \in \S,\nonumber
\end{align}
where $f$ is a suitably defined measure of error (called the \emph{loss} function) and $\S \subseteqq \R^n$ is a set that captures some sort of known structure that $x^*$ is  \emph{a priori} assumed to obey. A very common modeling assumption on $x^*$ is \emph{sparsity}, and $\S$ comprises the set of sparse vectors in some (known) basis representation. For example, smooth signals and images are (approximately) sparse when represented in the Fourier basis. This premise alleviates the ill-posed nature of the inverse problem, and in fact, it is well-known that accurate recovery of $x^*$ is possible if (i) the signal $x^*$ is sufficiently sparse, and (ii) measurement matrix $A$ satisfies certain algebraic conditions, such as the Restricted Isometry Property~\cite{candes2006compressive}. 

However, while being powerful from a computational standpoint, the sparsity prior has somewhat limited discriminatory capability. A sparse signal (or image) populated with random coefficients appears very distinct from the signals (or images) that abound in natural applications, and it is certainly true that nature exhibits far richer nonlinear structure than sparsity alone. This has spurred the development of estimation algorithms that use more refined priors, such as structured sparsity~\cite{modelcs,surveyEATCS}, dictionary models~\cite{elad2006image,aharon2006rm}, or bounded total variation~\cite{chan2006total}. While these priors often provide far better performance than using standard sparsity-based methods, they still suffer from the aforementioned limitations on modeling capability.
 %\cite{bora2017compressed} shows that the sparsity prior can be replaced by the generative Models and the solution to \ref{eq:setup2} can be found simply by Gradient Descent. However, the chances of getting stuck in local minima is high due to the extreme non-convex nature of the generator function involved.  

We focus on a newly emerging family of priors that are \emph{learned} from massive amounts of training data using a generative adversarial network (GAN)~\cite{goodfellow2014generative}. These priors are constructed by training the parameters of a certain neural network that simulates a nonlinear mapping from some latent parameter space of dimension $k \ll n$ to the high-dimensional ambient space $\R^n$. GANs have found remarkable applications in modeling image distributions~\cite{zhu2016generative,brock2016neural,chen2016infogan,zhao2016energy}, and a well-trained GAN closely captures the notion of a signal (or image) being `natural'~\cite{berthelot2017began}. Indeed, GAN-based neural network learning algorithms have been successfully employed to solve linear inverse problems such as image super-resolution and inpainting~\cite{yeh2016semantic, ledig2016photo}. However, these methods are mostly heuristic, and their theoretical properties are not yet well understood. Our goal in this paper is to take some initial steps towards a principled use of GAN priors for inverse problems.

\subsection{Our Contributions}

In this paper, we propose and analyze the well known \emph{projected gradient descent} (PGD) algorithm for solving \eqref{eq:cop}. We adopt a setting similar to the recent, seminal work of \cite{bora2017compressed}, and assume that the generator network (say, $G$) well approximates the high-dimensional probability distribution of the set $\S$, i.e., we expect that for each vector $x^*$ in $\S$, there exists a vector $\widehat{x} = G(\widehat{z})$ very close to $x^*$ in the support of distribution defined by $G$. 
The authors of~\cite{bora2017compressed} rigorously analyze the statistical properties of the minimizer of \eqref{eq:cop}. However, they do not explicitly discuss an \emph{algorithm} to perform this minimization. Instead, they re-parameterize \eqref{eq:cop} in terms of the latent variable $z$, and assume that gradient descent (or stochastic gradient descent) in the latent space provides an estimate of sufficiently high quality. However, if initialized incorrectly, (stochastic) gradient descent can get stuck in local minima, and therefore in practice their algorithm requires several restarts in order to provide good performance. Moreover, the rate of convergence of this method is not analyzed.

In contrast, we advocate using PGD to solve \eqref{eq:cop} directly in the ambient space. The high level idea in our approach is that through iterative projections, we are able to mitigate the effects of local minima and are able to explore the space outside the range of the generator $(G)$. 
%In this paper, we replace the sparsity prior with the pre-trained generative Adversarial Network. We present a novel Iterative Projection algorithm that we call \textbf{IP-GAN} which employs a pre-trained GAN to solve the inverse problem described in \ref{eq:setup1}.
Our procedure is depicted in Fig.~\ref{fig:intro1}. We choose a zero vector as our initial estimate($x_0$), and in each iteration, we update our estimate by following the standard gradient descent update rule (red arrow in Fig.~\ref{fig:intro1}), followed by projection of the output onto the span of generator $(G)$ (blue arrow in Fig.~\ref{fig:intro1}).

We support our PGD algorithm via a rigorous theoretical analysis. We show that the final estimate at the end of $T$ iterations is an approximate reconstruction of the original signal $x^*$, with very small reconstruction error; moreover, under certain sufficiency conditions on the linear operator $A$, PGD demonstrates linear convergence, meaning that $T = \log(1/\varepsilon)$ is sufficient to achieve $\varepsilon$-accuracy.
As further validation of our algorithm, we present a series of numerical results. We train two GANs: our first is a relatively simple two-layer generative model trained on the MNIST dataset~\cite{lecun1998gradient}; our second is a more complicated Deep Convolutional GAN \cite{radford2015unsupervised,taehoon2017} trained on the CelebA \cite{liu2015deep} dataset. In all experiments, we compare the performance of our algorithm with that of \cite{bora2017compressed} and a baseline algorithm using sparsity priors (specifically, the Lasso with a DCT basis). Our algorithm achieves the best performance both in terms of quantitative metrics (such as the structural similarity index) as well as visual quality.

%We also provide the theoretical proof of convergence for our algorithm by building on the mathematical basis developed by \cite{bora2017compressed}. 

%\subsection{Techniques}
% 1 --- GANs
%Recently, \cite{bora2017compressed} demonstrated the ability of Generative models such as VAEs and GANs to take place of sparsity priors in Compressed Sensing problems. Based on that, in this work we replace the traditional premise of sparsity with a structural model from a well-trained Generative Adversarial Network (GAN). GANs are powerful framework to learn the high dimensional data distributions of real signals such as images. It consists of two Networks, a Generator (G) and a Discriminator (D). Both the networks are trained simultaneously in adversarial manner -  Generator Network trained to generate realistic signals from the low dimensional random vector, usually drawn from the random normal or random uniform distribution $P(z)$; while Discriminator is trained to discriminate or classify the samples generated by the Generator as either 'real' or 'fake'. Goal of the Generator is to fool Discriminator by generating highly realistic samples. Thus, a well-trained GAN closely captures the notion of a signal being `natural' or `real'. This ability of GAN is leveraged to solve for an `realistic' estimate of $x^*$ by constraining the feasible region to the range of Generator (G). 

%2 ---- Iterative Projections
%\todo{review about iterative projections as a technique}

\begin{figure}
	\centering
	\def\svgwidth{\columnwidth}
	\input{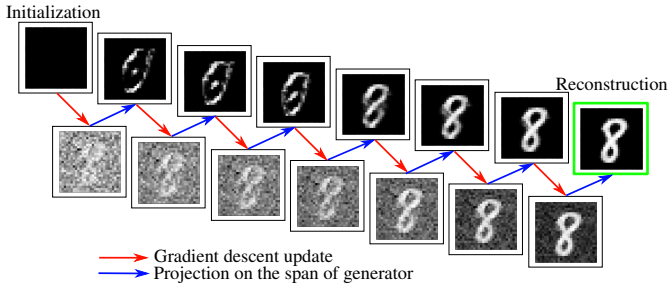}
	\caption{\emph{Illustration of our algorithm. Starting from a zero vector, we perform a gradient descent update step (red arrow) and projection step (blue arrow) alternatively to reach the final estimate.}}
	\label{fig:intro1}
\end{figure}

\subsection{Related Work}
%The linear inverse problem mentioned in Sec.\ref{sec:setup} is interdisciplinary in nature, which is evident through the vast literature available about similar inverse problems from diverse fields of study. All 

Approaches to solve linear inverse problems can be classified broadly in two categories. The approaches in the first category mainly use hand-crafted signal priors to distinguish `natural' signals from the infinite set of feasible solutions. The prior can be encoded in the form of either a constraint set (as in Eq.\ \ref{eq:cop}) or an extra regularization penalty. Several works (including \cite{donoho1995noising, xu2010image, dong2011image}) employ sparsity priors to solve inverse problems such as denoising, super-resolution and inpainting. In~\cite{elad2006image,aharon2006rm}, sparse and redundant dictionaries are learned for image denoising, whereas in \cite{rudin1992nonlinear,chambolle2004algorithm,chan2006total}, total variation is used as a regularizer. Despite their successful practical and theoretical results, all such hand-designed priors often fail to restrict the solution space only to natural images, and it is easily possible to generate signals satisfying the prior but do not resemble natural data.

The second category consists of learning-based methods involving the training of an end-to-end network mapping from the measurement space to the image space. Given a large dataset $x_i, i \in {1,2,..,N}$ and a measurement matrix $A$, the inverse mapping from $Ax_i$ to $x_i$ can be learned through a deep neural network training  \cite{lecun2015deep}. This approach is used in \cite{kulkarni2016reconnet,mousavi2015deep,mousavi2017learning,xu2014deep, dong2016image,kim2016accurate} to solve different inverse problems, and has met with considerable success. However, the major limitations are that a separate network is required for each new linear inverse problem; moreover, most of these methods lack concrete theoretical guarantees. The recent papers \cite{rick2017one,kelly2017deep} resolve this issue by training a single \emph{quasi-projection} operator to project each candidate solution on the manifold of natural images, and indeed in some sense is complementary to our approach. On the other hand, we train a \emph{generative} model that simulates the space of natural signals (or images) for a given application; moreover, our method can be rigorously analyzed. 

Recently, due to advances in adversarial training techniques \cite{goodfellow2014generative}, GANs have been explored as the powerful tool to solve challenging inverse problems. GANs can approximate the real data distribution closely, with visually striking results \cite{arjovsky2017wasserstein,berthelot2017began}. In \cite{yeh2016semantic, ledig2016photo}, GANs are used to solve the image inpainting and super-resolution problems respectively. The work closest to our work is the approach of leveraging GANs for compressive sensing~\cite{bora2017compressed}, which provides the basis for our work. Our method improves on the results of \cite{bora2017compressed} empirically, along with providing mathematical analysis of convergence.
%suggestions for related work:

%\begin{itemize}
%
%\item sparsity (candes/donoho compressed sensing paper papers), dictionary (elad and sapiro super-resolution), total variation (pick a few -> osher, romberg)
%
%\item neural network priors (GANs: original gan paper, couple of newer gan papers (wasserstein, Began)), OneNet (sankaranarayan end-to-end), LISTA (lecun et al), RNN/learning to learn papers (de frietas)
%
%\item Price et al (GANs for compressed sensing), projection onto GAN papers (one/two of them) 
%
%\end{itemize}
	\section{Algorithm and Main Results}

\subsection{Setup}
\label{sec:setup}
Let $\S \subseteq \R^n$ be the set of `natural' images in data space with a vector $x^* \in \S$. We consider an ill-posed linear inverse problem \eqref{eq:lip} with the linear operator $A$ being a Gaussian random matrix. For simplicity, we do not consider the additive noise term. 
To solve for $\widehat{x}$, we choose Euclidean measurement error as the loss function $f(\cdot)$ in Eqn. \eqref{eq:cop}. Therefore, given $y$ and $A$, we seek
\begin{align}
\widehat{x} = \argmin_{x \in \S}\|y-Ax\|^2.
\label{eq:setup2}
\end{align} 
All norms represented by $\|\cdot\|$ in this paper are Euclidean norms unless stated otherwise.
\subsection{Algorithm}

We train the generator $G : \R^k \rightarrow \R^n$ that maps a standard normal vector $z \in \R^k$ to the high dimensional sample space $G(z) \in \R^n$. We assume that our generator network well approximates the high-dimensional probability distribution of the set $\S$. With this assumption, we limit our search for $\widehat{x}$ only to the range of the generator function ($G(z)$). The function $G$ is assumed to be differentiable, and hence we use back-propagation for calculating the gradients of the loss functions involving $G$ for gradient descent updates.

The optimization problem in Eqn. \ref{eq:setup2} is similar to a least squares estimation problem, and a typical approach to solve such problems is to use gradient descent. However, the candidate solutions obtained after each gradient descent update need not represent a `natural' image and may not belong to set $\S$. We solve this limitation by projecting the candidate solution on the range of the generator function after each gradient descent update. Here, the projection of any vector $u$ on the generator is the image closest to $u$ in the span of the generator. %Just as traditional Compressive sensing setting uses the fact that all natural images are sparse in some known basis, here we employ the fact that for all the natural images or their approximates are present in the range of well-trained Generator function (G). 

Thus, in each iteration of our proposed algorithm \ref{alg:PGD-GAN}, two steps are performed in alternation: a gradient descent update step and a projection step. %$T$ is total number of iterations. 
\subsection{Gradient Descent Update}
The first step is simply an application of a gradient descent update rule on the loss function $f(\cdot)$ given as,
\[
f(x) \coloneqq \|y-Ax\|^2.
\] 
Thus, the gradient descent update at $t^{th}$ iteration is,
\[
w_t \leftarrow x_t + \eta A^T(y-Ax_t),
\] 
where $\eta$ is the learning rate.
\subsection{Projection Step}
In projection step, we aim to find an image from the span of the generator which is closest to our current estimate $w_t$. 
We define the projection operator $\mathcal{P}_G$ as follows:
\[
\mathcal{P}_G\left(w_t\right) \coloneqq G\left(\argmin_{z}f_{in}(z)\right),
\]
where the inner loss function is defined as,
\[
f_{in}(z) \coloneqq \|w_t - G(z)\|.
\]
We solve the inner optimization problem by running gradient descent with $T_{in}$ number of updates on $f_{in}(z)$. The learning rate $\eta_{in}$ is chosen empirically for this inner optimization. Though the inner loss function is highly non-convex due to the presence of $G$, we find empirically that the gradient descent (implemented via back-propagation) works very well. % in the task of minimizing the loss by searching for $\widehat{z}$ such that $G(\widehat{z})$ is a good approximation of the vector $w_t$. 
In each of the $T$ iterations, we run $T_{in}$ updates for calculating the projection. Therefore, $T \times T_{in}$ is the total number of gradient descent updates required in our approach.

\begin{algorithm}[t]
	\caption{\textsc{PGD-GAN}}
	\label{alg:PGD-GAN}
	\begin{algorithmic}[1]
	\State \textbf{Inputs:} $y$, $A$, $G$, $T$, \textbf{Output:}  $\widehat{x}$
	\State $x_0 \leftarrow \textbf{0}$ \hspace{17.8em} % $\triangleright$  \textbf{Initialization}
	\While {$t < T$}
	\State $w_t \leftarrow x_t + \eta A^T(y-Ax_t)$ \hspace{8.8em} %$\triangleright$ % \textbf{Gradient Descent update}
	\State $x_{t+1} \leftarrow \mathcal{P}_G(w_t) = G\left(\argmin_{z}\|w_t - G(z)\|\right)$ \hspace{0.6em} % $\triangleright$  \textbf{Projection}
		%\State estimate $z_l$ from $u = \exp(jtz_l)$ using 
	\State $t \leftarrow t+1$
	\EndWhile
	\State $\widehat{x} \leftarrow x_{T}$
	\end{algorithmic}
\end{algorithm}
	\subsection{Analysis}
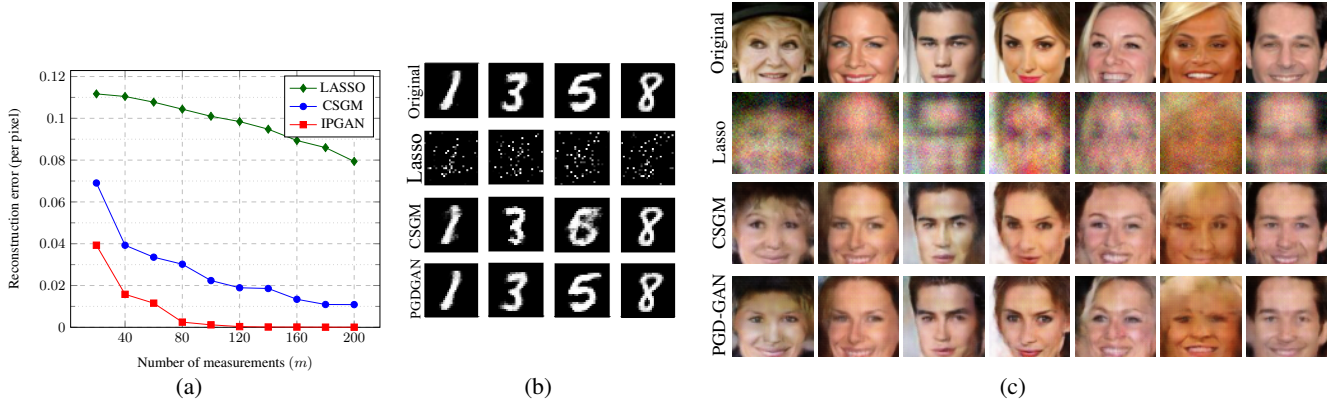
\begin{figure*}[!t]
	\begin{center}
		\begingroup
		\setlength{\tabcolsep}{4pt} % Default value: 6pt
		\renewcommand{\arraystretch}{1} % Default value: 1
		\begin{tabular}{ccc}      %{c@{\hskip .1pt}c@{\hskip .1pt}c}
			\raisebox{-0.6\height}{
			\begin{tikzpicture}[scale=0.6]
			\pgfplotsset{scaled y ticks=false}
			
			\begin{axis}[
			xlabel=Number of measurements $(m)$,
			ylabel=Reconstruction error (per pixel), 
			grid=both,
			minor y tick num=1,
			major grid style={dashed},
			minor grid style={dotted},
			ymin=0,
			xtick distance=40,
			yticklabel style={
				/pgf/number format/fixed,
				/pgf/number format/precision=5
			}] 
			\addplot[black!60!green, mark=diamond*, mark size=2.5] %lasso
			coordinates {
				(20,0.1116925674)
				(40,0.1104524388)
				(60,0.1077322946)
				(80,0.1043693087)
				(100,0.1009670696)
				(120,0.09845475989)
				(140,0.09478937129)
				(160,0.08939492804)
				(180,0.08597683983)
				(200,0.07938231549)}; 
			\addplot[blue, mark=*, mark size=2] %CSGM
			coordinates{
				
				(20,0.0690527727406)
				(40,  0.0392738807791)
				(60, 0.0335496354056)
				(80, 0.0301922520026)
				(100,  0.022366575626)
				(120, 0.0189172737835)
				(140, 0.0185702356315)
				(160, 0.0134109912651)
				(180, 0.0109117549589)
				(200,0.0108717565358)}; 
			\addplot[red, mark=square*, mark size=2] %IPGAN
			coordinates{
				(20,0.03928229894)
				(40,  0.01578411299)
				(60, 0.01155714582)
				(80, 0.002478766476)
				(100, 0.001179446696)
				(120, 0.0003373235479)
				(140, 0.0001405245319)
				(160,  1.25E-04)
				(180, 6.86E-05)
				(200,5.84E-05) };
			\legend{LASSO,CSGM,IPGAN} 
			\end{axis} 		
			\end{tikzpicture}}& 
			\raisebox{-0.5\height}{
				\includegraphics[width=0.20\linewidth]{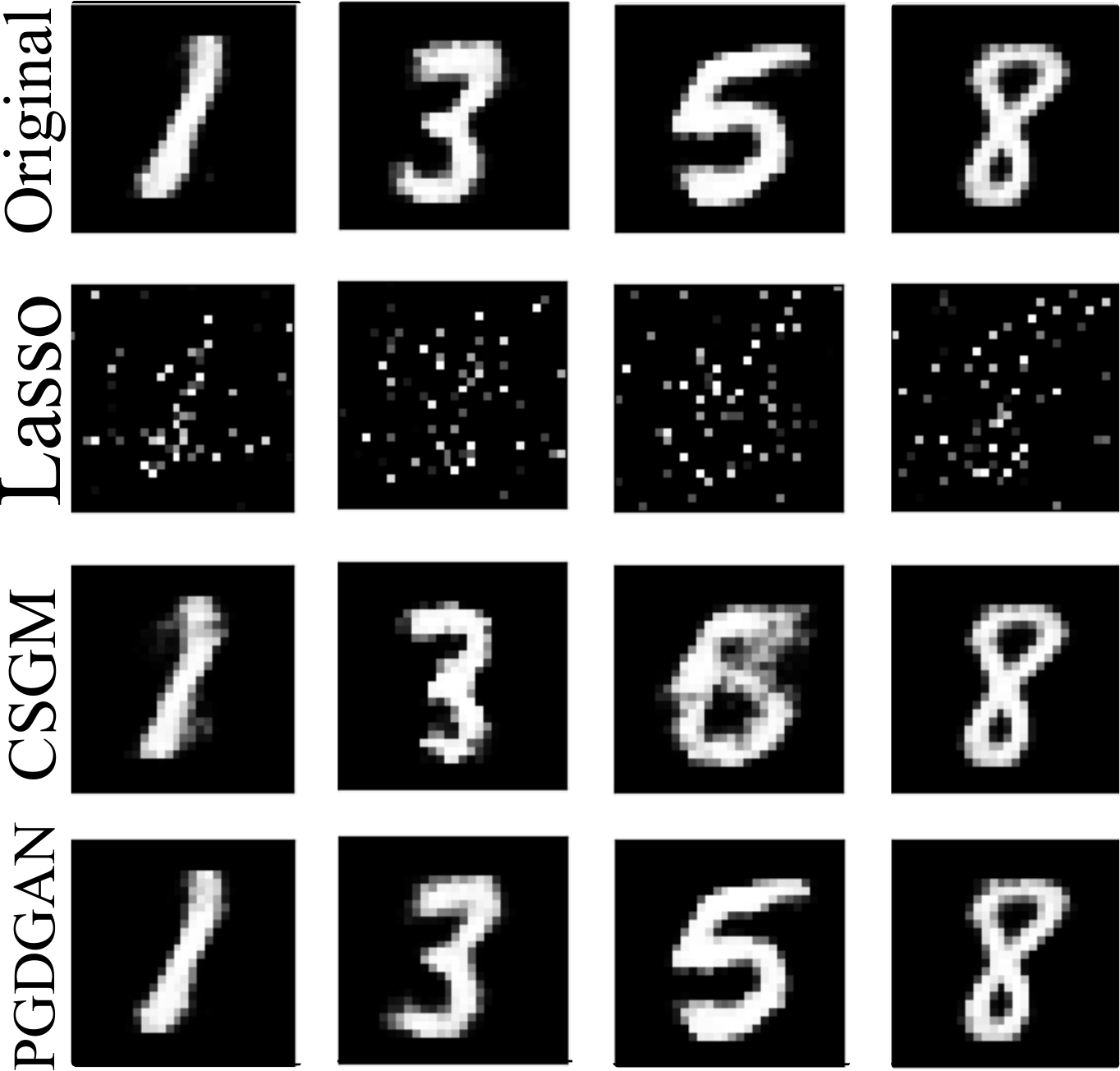}}&
				%	\begin{center}
					\begingroup
					\setlength{\tabcolsep}{1pt} % Default value: 6pt
					\renewcommand{\arraystretch}{1.2} % Default value: 1
					\begin{tabular}{cccccccc}      %{c@{\hskip .1pt}c@{\hskip .1pt}c}
						\begin{sideways}{\scriptsize ~~Original}\end{sideways}&
						%TO change image size, modify newcommand\s in main.tex
						\includegraphics[width=\s\linewidth]{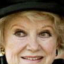}&
						\includegraphics[width=\s\linewidth]{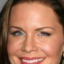}&
						\includegraphics[width=\s\linewidth]{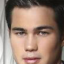}&
						\includegraphics[width=\s\linewidth]{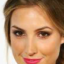}&
						\includegraphics[width=\s\linewidth]{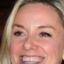}&
						\includegraphics[width=\s\linewidth]{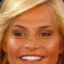}&
						\includegraphics[width=\s\linewidth]{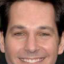}
						\\
						\begin{sideways}{\scriptsize ~~Lasso}\end{sideways}&
						\includegraphics[width=\s\linewidth]{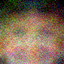}&
						\includegraphics[width=\s\linewidth]{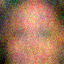}&
						\includegraphics[width=\s\linewidth]{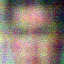}&
						\includegraphics[width=\s\linewidth]{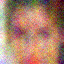}&
						\includegraphics[width=\s\linewidth]{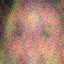}&
						\includegraphics[width=\s\linewidth]{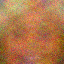}&
						\includegraphics[width=\s\linewidth]{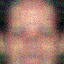}
						\\
						\begin{sideways}{\scriptsize ~~~CSGM}\end{sideways}&		
						\includegraphics[width=\s\linewidth]{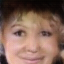}&
						\includegraphics[width=\s\linewidth]{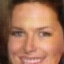}&
						\includegraphics[width=\s\linewidth]{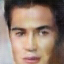}&
						\includegraphics[width=\s\linewidth]{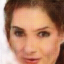}&
						\includegraphics[width=\s\linewidth]{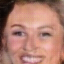}&
						\includegraphics[width=\s\linewidth]{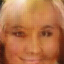}&
						\includegraphics[width=\s\linewidth]{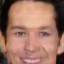}
						\\
						\begin{sideways}{\scriptsize PGD-GAN}\end{sideways}&
						\includegraphics[width=\s\linewidth]{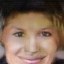}&
						\includegraphics[width=\s\linewidth]{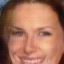}&
						\includegraphics[width=\s\linewidth]{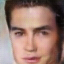}&
						\includegraphics[width=\s\linewidth]{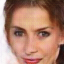}&
						\includegraphics[width=\s\linewidth]{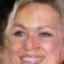}&
						\includegraphics[width=\s\linewidth]{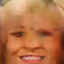}&
						\includegraphics[width=\s\linewidth]{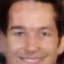}
					\end{tabular}
					\endgroup
					\\
			%	\end{center}\\
				(a) & (b) & (c)	
		\end{tabular}
		\endgroup
	\end{center}
	\caption{\emph{(a) Comparison of our algorithm with CSGM \cite{bora2017compressed} and Lasso on MNIST; (b) Reconstruction results with $m=100$ measurements; (c) Reconstruction results on celebA dataset with $m=1000$ measurements.}}
	\label{fig:mnist}
\end{figure*}

From standard compressive sensing theory, we know that conditions such as the restricted isometry property (RIP) on $A$ are sufficient to guarantee robust signal recovery. It has also been demonstrated that the RIP is a sufficient condition for recovery using iterative projections on manifolds \cite{shah2011iterative}. These conditions ensure that the operator $A$ preserves the uniqueness of the signal, i.e., the measurements corresponding to two different signals in the model would also be sufficiently different. In our case, we need to ensure that the difference vector of any two signals in the set $\S$ lies away from the nullspace of the matrix $A$. This condition is encoded via the \SREC~(Set Restricted Eigenvalue Condition) defined in \cite{bora2017compressed}. We slightly modify this condition and present it in the form of squared $l2$-norm : 
\begin{definition}
	Let $\S \in \R^n$. $A$ is $m \times n$ matrix. For parameters $\gamma > 0,~\delta \geq 0$, matrix $A$ is said to satisfy the \SREC$(\S, \gamma, \delta)$ if,
	\[
	\|A(x_1-x_2)\|^2 \geq \gamma \|x_1-x_2\|^2 - \delta,
	\]
	for $\forall x_1,x_2 \in \S$.
\end{definition}
Further, based on \cite{shah2011iterative,foucart2013}, we propose the following theorem about the convergence of our algorithm:
\begin{theorem}
Let $G: \R^k \rightarrow \R^n$ be a differentiable generator function with range $\S$. Let $A$ be a random Gaussian matrix with $A_{i,j}\sim N(0,1/m)$ such that it satisfies the \SREC$(\S, \gamma, \delta)$ with probability $1-p$, and has $\|Av\| \leq \rho\|v\|$ for every $v \in \R^n$ with probability $1-q$ with $\rho^2 \leq \gamma$.  %Let $A$ be a random Gaussian matrix with $A_{i,j}\sim N(0,1/m)$. 
Then, for every vector $x^* \in \S$, the sequence $\left(x_t\right)$ defined by the algorithm {PGD-GAN} [\ref{alg:PGD-GAN}] with $y = Ax^*$ converges to $x^*$ with probability at least $1-p-q$. 
\end{theorem}

\begin{proof}
	Define the squared error loss function $\psi(v) \coloneqq \|y - Av\|^2$. Then, we have:
	\begin{align*}
	& \psi(x_{t+1}) - \psi(x_t) \\
%	& = \| y - Ax_{t+1}\|^2 -  \| y- Ax_t\|^2, \\
%	& = [\|y\|^2 - 2\langle y, Ax_{t+1} \rangle + \|Ax_{t+1}\|^2] \\
%	&~~~~~- [\|y\|^2 - 2\langle y, Ax_t \rangle + \|Ax_t\|^2], \\
	& = \|Ax_{t+1}\|^2 - 2\langle y, Ax_{t+1} \rangle + 2\langle y, Ax_t \rangle - \|Ax_t\|^2, \\
%   & = \|Ax_{t+1} - Ax_t\|^2 + 2\langle Ax_t, Ax_{t+1} \rangle \\
%	&~~~~~~+ 2\langle y, Ax_t - Ax_{t+1} \rangle -2 \|Ax_t\|^2, \\	
	&= \|Ax_{t+1} - Ax_t\|^2 + 2\langle x_t-x_{t+1}, A^{T}A(x^*-x_t) \rangle.
	\end{align*}
	Substituting $y = Ax^*$ and rearranging yields,
	\begin{align}
	2\langle x_t - x_{t+1}, A^T(y-Ax_t) \rangle & = \psi(x_{t+1}) - \psi(x_t) \nonumber \\
	& - \|Ax_{t+1} - Ax_t\|^2.
	\label{eq:prf1}
	\end{align}
	Define:
	\begin{align*}
		w_t \coloneqq x_t + \eta A^T(y-Ax_t) = x_t + \eta A^TA(x^*-x_t)
	\end{align*}
	Then, by definition of the projection operator $P_G$, the vector $x_{t+1}$ is a better (or equally good) approximation to $w$
	as the true image $x^*$. Therefore, we have:
	\begin{align*}
	\|x_{t+1} - w_t\|^2 \leq \|x^* - w_t\|^2.
	\end{align*}
	Substituting for $w_t$ and expanding both sides, we get:
	\begin{align*}
%	& \|x_{t+1} - x_t - \eta A^T(y-Ax_t) \|^2 \\
%	& \leq \|x^* - x_t - \eta A^T(y-Ax_t)\|^2,~~~\text{or} \\
	& \|x_{t+1} - x_t\|^2 - 2\eta \langle x_{t+1}-x_t, A^T(y-Ax_t) \rangle \\
	& \leq \|x^* - x_t\|^2 - 2\eta \langle x^*-x_t, A^T(y-Ax_t) \rangle.
	\end{align*}
	Substituting $y = Ax^*$ and rearranging yields,
	\begin{align}
	& 2\langle x_t - x_{t+1}, A^T(y-Ax_t) \rangle \nonumber \\
	& \leq \frac{1}{\eta}\|x^* - x_t\|^2 - \frac{1}{\eta}\|x_{t+1} - x_t\|^2 - 2\psi(x_t).
	\label{eq:prf2}
	\end{align}
	We now use \ref{eq:prf1} and \ref{eq:prf2} to obtain,
%	\begin{align*}
%	& \psi(x_{t+1}) - \psi(x_t) - \|Ax_{t+1} - Ax_t\|^2 \\ 
%	& \leq \frac{1}{\eta}\|x^* - x_t\|^2 - \frac{1}{\eta}\|x_{t+1} - x_t\|^2 - 2\psi(x_t),~~\text{or}
%	\end{align*}
	\begin{align}
	& \psi(x_{t+1}) \leq \frac{1}{\eta}\|x^* - x_t\|^2 - \psi(x_t) \nonumber \\
	& - \left( \frac{1}{\eta}\|x_{t+1} - x_t\|^2 - \|Ax_{t+1} - Ax_t\|^2 \right).
	\label{eq:prf3}
	\end{align}
	Now, from the \SREC, we know that, 
	$$\|A (x_1 - x_2)\|^2 \geq \gamma \|x_1 - x_2 \|^2 - \delta.$$ %\todo{Add appendix with Verification of the SREC to be true for `squared' norm}
	As $x^*, x_t$ and $x_{t+1}$ are `natural' vectors,
	\begin{align}
	\frac{1}{\eta} \|x^* -x_t\|^2 \leq \frac{1}{\eta \gamma}\|y-Ax_t\|^2 + \frac{\delta}{\eta \gamma}.
	\label{eq:prf4}
	\end{align}
	Substituting \ref{eq:prf4} in \ref{eq:prf3},
	\begin{align*}
	& \psi(x_{t+1}) \leq \left(\frac{1}{\eta \gamma} -1\right) \psi(x_t)\\
	&~~~~~~~-\left( \frac{1}{\eta}\|x_{t+1} - x_t\|^2 - \|Ax_{t+1} - Ax_t\|^2 \right) + \frac{\delta}{\eta \gamma}.
	\end{align*}
	
	From our assumption that $\|Av\| \leq \rho \|v\|, \forall v \in \R^n$ with probability $1-q$, we write:
	$$\|Ax_{t+1} - Ax_t\|^2 \leq \rho^2\|x_{t+1} - x_t\|^2,$$
	$$ \|Ax_{t+1} - Ax_t\|^2 - \frac{1}{\eta}\|x_{t+1} - x_t\|^2 \leq \left(\rho^2 - \frac{1}{\eta}\right)\|x_{t+1} - x_t\|^2.$$
	Let us choose learning rate$(\eta)$ such that $\frac{1}{2\gamma}<\eta < \frac{1}{\gamma}$. We also have $\rho^2 \leq \gamma$. Combining both, we get $\rho^2 < \frac{1}{\eta}$, which makes the L.H.S. in the above equation negative. Therefore,
	$$\psi(x_{t+1}) \leq \left(\frac{1}{\eta \gamma} -1\right) \psi(x_t) + \frac{\delta}{\eta \gamma},$$
	where $\delta$ is inversely proportional to the number of measurements $m$ \cite{bora2017compressed}. Provided sufficient number of measurements, $\delta$ is small enough and can be ignored. Also, $\frac{1}{2\gamma}<\eta < \frac{1}{\gamma}$ yields,
	$$0<\left(\frac{1}{\eta \gamma} -1\right) < 1.$$
	Hence,
	\begin{align}
	\psi(x_{t+1}) \leq \alpha \psi(x_t); ~0< \alpha < 1,
	\label{eq:prf5}
	\end{align}
	with probability at least $1-p-q$.
\end{proof}

	\section{Models and Experiments}
In this section, we describe our experimental setup and report the performance comparisons of our algorithm with that of \cite{bora2017compressed} as well as the LASSO. We use two different GAN architectures and two different datasets in our experiments to show that our approach can work with variety of GAN architectures and datasets. 

In our experiments, we choose the entries of the matrix $A$ independently from a Gaussian distribution with zero mean and $1/m$ standard deviation. We ignore the presence of noise; however, our experiments can be replicated with additive Gaussian noise. We use a gradient descent optimizer keeping the total number of update steps $(T \times T_{in})$ fixed for both algorithms and doesn't allow random restarts.

In the first experiment, we use a very simple GAN model trained on the MNIST dataset, which is collection of $60,000$ handwritten digit images, each of size $28 \times 28$ \cite{lecun1998gradient}. In our GAN, both the generator and the discriminator are fully-connected neural networks with only one hidden layer. The generator consists of $20$ input neurons, $200$ hidden-layer neurons and $784$ output neurons, while the discriminator consists of $784$ input neurons, $128$ hidden layer neurons and $1$ output neuron. The size of the latent space is set to $k = 20$, i.e., the input to our generator is a standard normal vector $z \in R^{20}$.  We train the GAN using the method described in \cite{goodfellow2014generative}. We use the Adam optimizer \cite{kingma2014adam} with learning rate $0.001$ and mini-batch size $128$ for the training. 
 
We test the MNIST GAN with $10$ images taken from the span of generator to get rid of the representation error, and provide both quantitative and qualitative results. For PGD-GAN, because of the zero initialization, a high learning rate is required to get a meaningful output before passing it to the projection step. Therefore, we choose $\eta \geq 0.5$. The parameter $\eta_{in}$ is set to $0.01$ with $T=15$ and $T_{in}=200$. Thus, the total number of update steps is fixed to $3000$. Similarly, the algorithm of \cite{bora2017compressed} is tested with $3000$ updates and $\eta = 0.01$. For comparison, we use the reconstruction error $= \|\widehat{x}-x^*\|^2$. In Fig. \ref{fig:mnist}(a), we show the reconstruction error comparisons for increasing values of number of measurements. We observe that our algorithm performs better than the other two methods. Also, as the input images are chosen from the span of the generator itself, it is possible to get close to zero error with only $100$ measurements. Fig. \ref{fig:mnist}(b) depicts reconstruction results for selected MNIST images.
 
The second set of our experiments are performed on a Deep Convolutional GAN (DCGAN) trained on the celebA dataset, which contains more than $200,000$ face images of celebrities \cite{liu2015deep}. We use a pre-trained DCGAN model, which was made available by \cite{bora2017compressed}. Thus, the details of the model and training are the same as described in \cite{bora2017compressed}. The dimension of latent space for DCGAN is $k=100$. We report the results on a held out test dataset, unseen by the GAN at the time of training. Total number of updates is set to $1000$, with $T = 10$ and $T_{in}=100$. Learning rates for PGD-GAN are set as $\eta = 0.5$ and $\eta_{in}=0.1$. The algorithm of \cite{bora2017compressed} is run with $\eta = 0.1$ and $1000$ update steps. Image reconstruction results from $m=1000$ measurements with our algorithm are displayed in Fig. \ref{fig:mnist}(c). We observe that our algorithm produces better reconstructions compared to the other baselines.
	{{
	\footnotesize
	\bibliographystyle{IEEEbib}
	\bibliography{./common/chinbiblio,./common/csbib,./common/mrsbiblio,./common/vsbib,./common/kernels}

\begin{thebibliography}{10}

\bibitem{candes2006compressive}
E.~Cand{\`e}s et~al.,
\newblock ``Compressive sampling,''
\newblock in {\em Proc. of the intl. congress of math.} Madrid, Spain, 2006,
  vol.~3, pp. 1433--1452.

\bibitem{candes2006stable}
E.~Candes, J.~Romberg, and T.~Tao,
\newblock ``Stable signal recovery from incomplete and inaccurate
  measurements,''
\newblock {\em Comm. on pure and appl. math.}, vol. 59, no. 8, pp. 1207--1223,
  2006.

\bibitem{donoho2006compressed}
D.~Donoho,
\newblock ``Compressed sensing,''
\newblock {\em IEEE Trans. Inform. Theory}, vol. 52, no. 4, pp. 1289--1306,
  2006.

\bibitem{modelcs}
R.~Baraniuk, V.~Cevher, M.~Duarte, and C.~Hegde,
\newblock ``Model-based compressive sensing,''
\newblock {\em IEEE Trans. Inform. Theory}, vol. 56, no. 4, pp. 1982--2001,
  Apr. 2010.

\bibitem{surveyEATCS}
C.~Hegde, P.~Indyk, and L.~Schmidt,
\newblock ``Fast algorithms for structured sparsity,''
\newblock {\em Bulletin of the EATCS}, vol. 1, no. 117, pp. 197--228, Oct.
  2015.

\bibitem{elad2006image}
M.~Elad and M.~Aharon,
\newblock ``Image denoising via sparse and redundant representations over
  learned dictionaries,''
\newblock {\em IEEE Trans. Image Processing}, vol. 15, no. 12, pp. 3736--3745,
  2006.

\bibitem{aharon2006rm}
M.~Aharon, M.~Elad, and A.~Bruckstein,
\newblock ``$ rm k $-svd: An algorithm for designing overcomplete dictionaries
  for sparse representation,''
\newblock {\em IEEE Trans. Signal Processing}, vol. 54, no. 11, pp. 4311--4322,
  2006.

\bibitem{chan2006total}
T.~Chan, J.~Shen, and H.~Zhou,
\newblock ``Total variation wavelet inpainting,''
\newblock {\em Jour. of Math. imaging and Vision}, vol. 25, no. 1, pp.
  107--125, 2006.

\bibitem{goodfellow2014generative}
I.~Goodfellow, J.~Pouget-Abadie, M.~Mirza, B.~Xu, D.~Warde-Farley, S.~Ozair,
  A.~Courville, and Y.~Bengio,
\newblock ``Generative adversarial nets,''
\newblock in {\em Proc. Adv. in Neural Processing Systems (NIPS)}, 2014, pp.
  2672--2680.

\bibitem{zhu2016generative}
J.~Zhu, P.~Kr{\"a}henb{\"u}hl, E.~Shechtman, and A.~Efros,
\newblock ``Generative visual manipulation on the natural image manifold,''
\newblock in {\em Proc. European Conf. Comp. Vision (ECCV)}, 2016.

\bibitem{brock2016neural}
A.~Brock, T.~Lim, J.~Ritchie, and N.~Weston,
\newblock ``Neural photo editing with introspective adversarial networks,''
\newblock {\em arXiv preprint arXiv:1609.07093}, 2016.

\bibitem{chen2016infogan}
X.~Chen, Y.~Duan, R.~Houthooft, J.~Schulman, I.~Sutskever, and P.~Abbeel,
\newblock ``Infogan: Interpretable representation learning by information
  maximizing generative adversarial nets,''
\newblock in {\em Proc. Adv. in Neural Processing Systems (NIPS)}, 2016, pp.
  2172--2180.

\bibitem{zhao2016energy}
J.~Zhao, M.~Mathieu, and Y.~LeCun,
\newblock ``Energy-based generative adversarial network,''
\newblock {\em arXiv preprint arXiv:1609.03126}, 2016.

\bibitem{berthelot2017began}
D.~Berthelot, T.~Schumm, and L.~Metz,
\newblock ``Began: Boundary equilibrium generative adversarial networks,''
\newblock {\em arXiv preprint arXiv:1703.10717}, 2017.

\bibitem{yeh2016semantic}
R.~Yeh, C.~Chen, T.~Lim, M.~Hasegawa-Johnson, and M.~Do,
\newblock ``Semantic image inpainting with perceptual and contextual losses,''
\newblock {\em arXiv preprint arXiv:1607.07539}, 2016.

\bibitem{ledig2016photo}
C.~Ledig, L.~Theis, F.~Husz{\'a}r, J.~Caballero, A.~Cunningham, A.~Acosta,
  A.~Aitken, A.~Tejani, J.~Totz, Z.~Wang, et~al.,
\newblock ``Photo-realistic single image super-resolution using a generative
  adversarial network,''
\newblock {\em Proc. IEEE Conf. Comp. Vision and Pattern Recog. (CVPR)}, pp.
  105--114, 2017.

\bibitem{bora2017compressed}
A.~Bora, A.~Jalal, E.~Price, and A.~Dimakis,
\newblock ``Compressed sensing using generative models,''
\newblock {\em Proc. Int. Conf. Machine Learning}, 2017.

\bibitem{lecun1998gradient}
Y.~LeCun, L.~on Bottou, Y.~Bengio, and P.~Haffner,
\newblock ``Gradient-based learning applied to document recognition,''
\newblock {\em Proc. of the IEEE}, vol. 86, no. 11, pp. 2278--2324, 1998.

\bibitem{radford2015unsupervised}
A.~Radford, L.~Metz, and S.~Chintala,
\newblock ``Unsupervised representation learning with deep convolutional
  generative adversarial networks,''
\newblock {\em Proc. Int. Conf. Learning Representations (ICLR)}, 2016.

\bibitem{taehoon2017}
K.~Taehoon,
\newblock ``A tensorflow implementation of “deep convolutional generative
  adversarial networks”,'' 2017.

\bibitem{liu2015deep}
Z.~Liu, P.~Luo, X.~Wang, and X.~Tang,
\newblock ``Deep learning face attributes in the wild,''
\newblock in {\em Proc. of the IEEE Intl. Conf. on Comp. Vision}, 2015, pp.
  3730--3738.

\bibitem{donoho1995noising}
D.~Donoho,
\newblock ``De-noising by soft-thresholding,''
\newblock {\em IEEE Trans. Inform. Theory}, vol. 41, no. 3, pp. 613--627, 1995.

\bibitem{xu2010image}
Z.~Xu and J.~Sun,
\newblock ``Image inpainting by patch propagation using patch sparsity,''
\newblock {\em IEEE Trans. Image Processing}, vol. 19, no. 5, pp. 1153--1165,
  2010.

\bibitem{dong2011image}
W.~Dong, L.~Zhang, G.~Shi, and X.~Wu,
\newblock ``Image deblurring and super-resolution by adaptive sparse domain
  selection and adaptive regularization,''
\newblock {\em IEEE Trans. Image Processing}, vol. 20, no. 7, pp. 1838--1857,
  2011.

\bibitem{rudin1992nonlinear}
L.~Rudin, S.~Osher, and E.~Fatemi,
\newblock ``Nonlinear total variation based noise removal algorithms,''
\newblock {\em Physica D: Nonlinear Phenomena}, vol. 60, no. 1-4, pp. 259--268,
  1992.

\bibitem{chambolle2004algorithm}
A.~Chambolle,
\newblock ``An algorithm for total variation minimization and applications,''
\newblock {\em Jour. of Math. imaging and vision}, vol. 20, no. 1, pp. 89--97,
  2004.

\bibitem{lecun2015deep}
Y.~LeCun, Y.~Bengio, and G.~Hinton,
\newblock ``Deep learning,''
\newblock {\em Nature}, vol. 521, no. 7553, pp. 436--444, 2015.

\bibitem{kulkarni2016reconnet}
K.~Kulkarni, S.~Lohit, P.~Turaga, R.~Kerviche, and A.~Ashok,
\newblock ``Reconnet: Non-iterative reconstruction of images from compressively
  sensed measurements,''
\newblock in {\em Proc. IEEE Conf. Comp. Vision and Pattern Recog. (CVPR)},
  2016, pp. 449--458.

\bibitem{mousavi2015deep}
A.~Mousavi, A.~Patel, and R.~Baraniuk,
\newblock ``A deep learning approach to structured signal recovery,''
\newblock in {\em Proc. Allerton Conf. Communication, Control, and Computing},
  2015, pp. 1336--1343.

\bibitem{mousavi2017learning}
A.~Mousavi and R.~Baraniuk,
\newblock ``Learning to invert: Signal recovery via deep convolutional
  networks,''
\newblock {\em Proc. IEEE Int. Conf. Acoust., Speech, and Signal Processing
  (ICASSP)}, 2017.

\bibitem{xu2014deep}
L.~Xu, J.~Ren, C.~Liu, and J.~Jia,
\newblock ``Deep convolutional neural network for image deconvolution,''
\newblock in {\em Proc. Adv. in Neural Processing Systems (NIPS)}, 2014, pp.
  1790--1798.

\bibitem{dong2016image}
C.~Dong, C.~Loy, K.~He, and X.~Tang,
\newblock ``Image super-resolution using deep convolutional networks,''
\newblock {\em IEEE Trans. Pattern Anal. Machine Intell.}, vol. 38, no. 2, pp.
  295--307, 2016.

\bibitem{kim2016accurate}
J.~Kim, J.~Kwon~Lee, and K.~Mu~Lee,
\newblock ``Accurate image super-resolution using very deep convolutional
  networks,''
\newblock in {\em Proc. IEEE Conf. Comp. Vision and Pattern Recog. (CVPR)},
  2016, pp. 1646--1654.

\bibitem{rick2017one}
J.~Rick~Chang, C.~Li, B.~Poczos, B.~Vijaya~Kumar, and A.~Sankaranarayanan,
\newblock ``One network to solve them all--solving linear inverse problems
  using deep projection models,''
\newblock in {\em Proc. IEEE Conf. Comp. Vision and Pattern Recog. (CVPR)},
  2017, pp. 5888--5897.

\bibitem{kelly2017deep}
B.~Kelly, T.~Matthews, and M.~Anastasio,
\newblock ``Deep learning-guided image reconstruction from incomplete data,''
\newblock {\em arXiv preprint arXiv:1709.00584}, 2017.

\bibitem{arjovsky2017wasserstein}
M.~Arjovsky, S.~Chintala, and L.~Bottou,
\newblock ``Wasserstein gan,''
\newblock {\em arXiv preprint arXiv:1701.07875}, 2017.

\bibitem{shah2011iterative}
P.~Shah and V.~Chandrasekaran,
\newblock ``Iterative projections for signal identification on manifolds:
  Global recovery guarantees,''
\newblock in {\em Proc. Allerton Conf. Communication, Control, and Computing},
  2011, pp. 760--767.

\bibitem{foucart2013}
S.~Foucart and H.~Rauhut,
\newblock {\em A mathematical introduction to compressive sensing}, vol.~1,
\newblock Springer.

\bibitem{kingma2014adam}
D.~Kingma and J.~Ba,
\newblock ``Adam: A method for stochastic optimization,''
\newblock {\em arXiv preprint arXiv:1412.6980}, 2014.

\end{thebibliography}
	}
	}

\end{document}